\documentclass{article}
\usepackage{times,uweetr}
\usepackage{graphicx}
\usepackage{amsmath, amssymb, amsthm}
\usepackage{algorithm}
\usepackage{algorithmic}
\usepackage{wrapfig}
\usepackage{url}
\usepackage[numbers]{natbib}

\newtheorem{theorem}{Theorem}
\newtheorem{lemma}{Lemma}
\newtheorem{corollary}{Corollary}
\newcommand{\argmax}[1]{{\hbox{$\underset{#1}{\operatorname{argmax}}\;$}}}

\begin{document}

\title{Average-Case Active Learning with Costs}

\author{Andrew Guillory 
\thanks{This material is based upon work supported by the
National Science Foundation under grant IIS-0535100 and
by an ONR MURI grant N000140510388} \\
Computer Science and Engineering \\
University of Washington \\
\tt{guillory@cs.washington.edu} \\
\and Jeff Bilmes \\
Electrical Engineering \\
University of Washington \\
\tt{bilmes@ee.washington.edu}}

\reportmonth{May}
\reportyear{2009}
\reportnumber{0005}

%\author{Andrew Guillory 
%\thanks{This material is based upon work supported by the
%National Science Foundation under grant IIS-0535100 and
%by an ONR MURI grant N000140510388}
%\inst{1} \and Jeff Bilmes \inst{2}}

%\institute{
%Computer Science and Engineering \\
%University of Washington \\
%\email{guillory@cs.washington.edu} \and
%Electrical Engineering \\
%University of Washington \\
%\email{bilmes@ee.washington.edu}
%}

\makecover

\maketitle

\begin{abstract}
We analyze the expected cost of a greedy active learning algorithm.
Our analysis extends previous work to a more general setting in
which different queries have different costs.  Moreover, queries may
have more than two possible responses and the distribution over
hypotheses may be non uniform.  Specific applications include active
learning with label costs, active learning for multiclass and
partial label queries, and batch mode active learning.  We also discuss an
approximate version of interest when there are very many queries.
\end{abstract}

\section{Motivation}
We first motivate the problem by describing it informally.
Imagine two people are playing a variation of twenty questions.  Player 1
selects an object from a finite set, and it is up to
player 2 to identify the selected object by asking questions chosen
from a finite set.  We assume for every object and
every question the answer is unambiguous: each question maps each
object to a single answer.  Furthermore, each question has associated
with it a cost, and the goal of player 2 is to identify
the selected object using a sequence of questions with minimal cost.
There is no restriction that the questions are yes or no questions.  
Presumably, complicated, more
specific questions have greater costs.  It doesn't violate the rules
to include a single question enumerating all the objects (Is the
object a dog or a cat or an apple or...), but for the game to be
interesting it should be possible to identify the object using a
sequence of less costly questions. 

With player 1 the human expert and player 2 the learning algorithm, we
can think of active learning as a game of twenty questions.  The set
of objects is the hypothesis class, the selected object is the optimal
hypothesis with respect to a training set, and the questions available
to player 2 are label queries for data points in the finite sized
training set.  Assuming the data set is separable, label queries are
unambiguous questions (i.e. each question has an unambiguous answer).
By restricting the hypothesis class to be a set of possible labellings
of the training set (i.e. the effective hypothesis class for some
other possibly infinite hypothesis class), we can also ensure there is
a unique zero-error hypothesis.  If we set all question costs to 1, we
recover the traditional active learning problem of identifying the
target hypothesis using a minimal number of labels.

However, this framework is also general enough to cover a variety of active
learning scenarios outside of traditional binary classification.
\begin{itemize}
\item \textbf{Active learning with label costs} If different data
points are more or less costly to label, we can model these
differences using non uniform label costs.  For example, if a longer
document takes longer to label than a shorter document,
we can make costs proportional to document length.  The goal is
then to identify the optimal hypothesis as quickly as possible as
opposed to using as few labels as possible. 
This notion of label cost is different than the often studied
notion of misclassification cost.  Label cost refers to the cost of
acquiring a label at training time where misclassification cost
refers to the cost of incorrectly predicting a label at test time.
\item \textbf{Active learning for multiclass and partial
label queries}
We can directly ask for the label of a point 
(Is the label of this point ``a'', ``b'', or
``c''?), or we can ask less specific questions about the label (Is the label
of this point ``a'' or some other label?).  We can also mix these question 
types, presumably making less specific questions less costly.  These
kinds of partial label queries are particularly important when examples
have structured labels.  In a parsing problem, a partial
label query could ask for the portion of a parse tree corresponding
to a small phrase in a long sentence.
\item \textbf{Batch mode active learning} Questions can also be
queries for multiple labels.  In the extreme case, there can be a
question corresponding to every subset of possible single data point
questions.  Batch label queries only help the algorithm reduce total
label cost if the cost of querying for a batch of labels is in some
cases less than the of sum of the corresponding individual label
costs.  This is the case if there is a constant
additive cost overhead associated with asking a question or if we
want to minimize time spent labeling and there
are multiple labelers who can label examples in parallel.
\end{itemize}
Beyond these specific examples, this setting applies to any active
learning problem for which different user interactions have different
costs and are unambiguous as we have defined.  For example, we can ask
questions concerning the percentage of positive and negative examples
according to the optimal classifier (Does the optimal classifier label
more than half of the data set positive?).  This abstract setting also
has applications outside of machine learning.
\begin{itemize}
\item \textbf{Information Retrieval} We can think of a question asking
strategy as an index into the set of objects which can then be used
for search. If we make the cost of a question the expected
computational cost of computing the answer for a given object, then
a question asking strategy with low cost corresponds to an index
with fast search time.  For example, if objects correspond to points
in $\Re^n$ and questions correspond to axis aligned hyperplanes, a
question asking strategy is a $kd$-tree.
\item \textbf{Compression} A question asking strategy produces a
unique sequence of responses for each object.  If we make the cost
of a question the log of the number of possible
responses to that question,
then a question asking strategy with low cost
corresponds to a code book for the set of objects 
with small code length \citep{coverthomas}.
\end{itemize}
Interpreted in this way, active learning, information retrieval, and
compression can be thought of as variations of the same problem
in which we minimize interaction cost, computation cost, and code length
respectively.  

In this work we consider this general problem for average-case cost.
The object is selected at random and the goal is to minimize the
expected cost of identifying the selected object.  The distribution from which
the object is drawn is known but may not be uniform.  Previous work
\citep{optimalsplit, greedy, approxoptimal, 
decisiontreesentity, approxdecision}
has shown simple greedy algorithms are approximately optimal
in certain more restrictive settings.  We extend these results to our
more general setting.  

\section{Preliminaries}

\begin{figure}[t]
\label{tree}
\centering
\includegraphics[width=.5\textwidth]{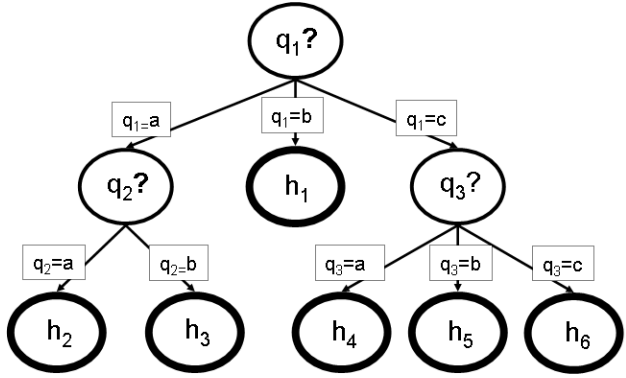}
\caption{Decision tree view of active learning.  Internal nodes
are questions (label queries), branches are answers (label values),
and leaves are target objects (hypotheses).  The cost of identifying
a target object is the sum of the question costs along the
path from the root to that object.}
\end{figure}

We first review the main result of \citet{greedy} which our first
bound extends.  We assume we have a finite set of objects (for example
hypotheses) $H$ with $|H|=n$.  A randomly chosen $h^* \in H$ is our
target object with a known positive $\pi(h)$ defining the
distribution over $H$ by which $h^*$ is drawn.  We assume $\min_h
\pi(h) > 0$ and $|H|>1$.  We also assume there is a finite set of
questions $q_1, q_2, ... q_m$ each of which has a positive cost $c_1,
c_2,... c_m$.  Each question $q_i$ maps each object to a response from
a finite set of answers $A \triangleq \bigcup_{h,i} \{ q_i(h) \}$
and asking $q_i$ reveals $q_i(h^*)$, eliminating from consideration
all objects $h$ for which $q_i(h) \neq q_i(h^*)$.  An active learning
algorithm continues asking questions until $h^*$ has been identified
(i.e. we have eliminated all but one of the elements from $H$).  We
assume this is possible for any element in $H$.  The goal of the
learning algorithm is to identify $h^*$ with questions incurring as
little cost as possible.  Our result bounds the expected cost of
identifying $h^*$.

We assume that the distribution $\pi$, the hypothesis class $H$, the
questions $q_i$, and the costs $c_i$ are known.  Any deterministic
question asking strategy (e.g. a deterministic active learning
algorithm taking in this known information) produces a decision tree
in which internal nodes are questions and the leaves are elements of
$H$.  The \emph{cost} of a \emph{query tree} $T$ with respect to a
distribution $\pi$, $C(T, \pi)$, is defined to be the expected cost of
identifying $h^*$ when $h^*$ is chosen according to $\pi$.  We can
write $C(T, \pi)$ as $C(T, \pi) = \sum_{h \in H} \pi(h) c_T(h)$ where
$c_T(h)$ is the cost to identify $h$ as the target object.  $c_T(h)$
is simply the sum of the costs of the questions along the path from
the root of $T$ to $h$.  We define $\pi_S$ to be $\pi$ restricted and
normalized w.r.t.\ $S$.  For $s \in S$, $\pi_S(s) = \pi(s) / \pi(S)$,
and for $s \notin S$, $\pi_S(s) = 0$.  Tree cost decomposes nicely.
\begin{lemma}
For any tree $T$ and any $S = \bigcup_{i} S^i$ with 
$\forall_{i,j} S^i \cap S^j = \emptyset$, $S \neq \emptyset$
\[ C(T, \pi_S) = \sum_{i} \pi_S(S^i) C(T, \pi_{S^i}) \]
\label{decomplem}
\end{lemma}

We define the \emph{version space} to be the subset of $H$
consistent with the answers we have received so far.  Questions eliminate
elements from the version space.  For a question $q_i$ and a particular
version space $S \subseteq H$, we define $S^j \triangleq \lbrace s
\in S : q_i(s)=j \rbrace$.  With this notation the dependence on $q_i$
is suppressed but understood by context.  As shorthand, for a
distribution $\pi$ we define $\pi(S) = \sum_{s \in S} \pi(s)$.  On
average, asking question $q_i$ \emph{shrinks} the absolute mass of $S$ with
respect to a distribution $\pi$ by
\[ \begin{split}
\Delta_i(S,\pi) & \triangleq \sum_{j \in A} \frac{\pi(S^j)}{\pi(S)}(\sum_{k
  \neq j} \pi(S^k)) 
= \pi(S) - \sum_{j \in A}
\frac{\pi(S^j)^2}{\pi(S)} 
\end{split} \] 
We call this quantity the
\emph{shrinkage} of $q_i$ with respect to $(S,\pi)$.  We note
$\Delta_i(S,\pi)$ is only defined if $\pi(S) > 0$.  If $q_i$ has cost
$c_i$, we call $\frac{\Delta_i(S,\pi)}{c_i}$ the \emph{shrinkage-cost
  ratio} of $q_i$ with respect to $(S,\pi)$.  

\begin{algorithm}[t]
\caption{Cost Sensitive Greedy Algorithm}
\begin{algorithmic}[1]
\STATE $S \Leftarrow H$
\REPEAT
\STATE $i = \argmax{i} \Delta_i(S, \pi_S) / c_i$
\STATE $S \Leftarrow \lbrace s \in S : q_i(s) = q_i(h^*) \rbrace$
\UNTIL{$|S| = 1$}
\end{algorithmic}
\label{greedyalg}
\end{algorithm}

In previous work \citep{greedy, approxoptimal, 
decisiontreesentity}, 
the greedy algorithm analyzed is the algorithm that
at each step chooses the question $q_i$ that maximizes the shrinkage
with respect to the current version space $\Delta_i(S, \pi_S)$.  In
our generalized setting, we define the \emph{cost sensitive greedy
  algorithm} to be the active learning algorithm which at each step
asks the question with the largest shrinkage-cost ratio $\Delta_i(S,
\pi_S) / c_i$ where $S$ is the current version space.  We call the
tree generated by this method the greedy query tree.  See Algorithm
\ref{greedyalg}.  \citet{approxoptimal} also analyzed a cost-sensitive
method for the restricted case of 
questions with two responses and uniform $\pi$, and our
method is equivalent to theirs in this case.  The main result of
\citet{greedy} is that, on average, with unit costs and yes/no
questions, the greedy strategy is not much worse than any other
strategy.  We repeat this result here.
\begin{theorem} Theorem 3 \citep{greedy}  If $|A|=2$ and  
$\forall i$ $c_i = 1$, then for any $\pi$ the greedy query tree 
$T^g$ has cost at most 
\[ C(T^g, \pi) \leq 4 C^* \ln 1/(\min_{h \in H} \pi(h)) \] where $C^*
= \min_T C(T, \pi)$. \label{oldthm} \end{theorem} 
For a uniform, 
$\pi$, the log term becomes $\ln |H|$, so the approximation factor
grows with the log of the number of objects.  In the non uniform case,
the greedy algorithm can do significantly worse.  However,
\citet{optimalsplit} and \citet{decisiontreesentity} show a simple
rounding method can be used to remove dependence on $\pi$ .  We first
give an extension to Theorem \ref{oldthm} to our more general setting.
We then show we how to remove dependence on $\pi$ using a similar
rounding method.  Interestingly, in our setting this rounding method
introduces a dependence on the costs, so neither bound is strictly
better although together they generalize all previous results.

\section{Cost Independent Bound}
\begin{theorem}
For any $\pi$ the greedy query tree $T^g$ has cost 
at most 
\[ C(T^g, \pi) \leq 12 C^* \ln 1/(\min_{h \in H} \pi(h)) \]
where $C^* \triangleq \min_T C(T, \pi)$. \label{mainthm} \end{theorem}
What is perhaps surprising about this bound is that the quality of
approximation does not depend on the costs themselves. 
The proof follows part of the strategy used by \citet{greedy}.  The general
approach is to show that if the average cost of some question
tree is low, then there must be at least one question with
high shrinkage-cost ratio.  We then use this to
form the basis of an inductive argument.  However, 
this simple argument fails when only a few objects have high probability mass.

We start by showing the shrinkage of
$q_i$ monotonically decreases as we eliminate elements from $S$.
\begin{lemma}Extension of Lemma 6 \citep{greedy} to non binary queries.  
If $T \subseteq S \subseteq H$,
and $T \neq \emptyset$ then, $\forall i, \pi$,  $\Delta_i(T,
\pi) \leq \Delta_i(S, \pi)$. \label{mondeclma} \end{lemma}
\begin{proof}
For $|S| = 1$ the result is immediate since $|T| \geq 1$ and therefore $S = T$.
We show that if $|S| > 2$, removing any single element $a \in S \setminus T$ 
from $S$ does not increase $\Delta_i(S, \pi)$.  
The lemma then follows since we can
remove all of $S \setminus T$ from $S$ an element at a time. 
Assume w.l.o.g.\ $a \in S^k$ for some $k$. Here let $A' 
\triangleq A \setminus \lbrace k \rbrace$
\begin{align*} 
\Delta_i(S-\lbrace a \rbrace, \pi)  = 
\frac{(\pi(S^k) - \pi(a))(\pi(S) - \pi(S^k))}{\pi(S) - \pi(a)}  
+ \sum_{j \in A'} \frac{\pi(S^j)(\pi(S)-\pi(S^j)-\pi(a))}{\pi(S)-\pi(a)}
\end{align*}
We show that this is term by term less than or equal to
\begin{align*}
\Delta_i(S, \pi) & =  \frac{\pi(S^k)(\pi(S) - \pi(S^k))}{\pi(S)} 
+  \sum_{j \in A'} \frac{\pi(S^j)(\pi(S) - \pi(S^j))}{\pi(S)}
\end{align*}
For the first term 
\[ \frac{(\pi(S^k) - \pi(a))(\pi(S) - \pi(S^k))}{\pi(S) - \pi(a)} 
\leq \frac{\pi(S^k)(\pi(S) - \pi(S^k))}{\pi(S)}\]
because $\pi(S) \geq \pi(S^k)$ and $\pi(a) \geq 0$.
For any other term in the summation,
\[ \frac{\pi(S^j)(\pi(S) - \pi(S^j) - \pi(a)))}{\pi(S) - \pi(a)} 
\leq \frac{\pi(S^j)(\pi(S) - \pi(S^j))}{\pi(S)}\]
because $\pi(S) - \pi(S^j) \geq \pi(a) \geq 0$ and $\pi(S) > \pi(a)$.
%\qed \end{proof}
\end{proof}
Obviously, the same result holds when we consider shrinkage-cost
ratios.
\begin{corollary} If $T \subseteq S \subseteq H$,
and $T \neq \emptyset$ then for any $i, \pi$,  $\Delta_i(T,
\pi) / c_i \leq \Delta_i(S, \pi) / c_i$. \label{mondeccor} \end{corollary}

We define the \emph{
collision probability} of a distribution $v$ over $Z$ to be
$\mathsf{CP}(v) \triangleq \sum_{z \in Z} v(z)^2 $ 
This is exactly the probability two samples  from $v$ will be the same
and quantifies the extent to which mass is concentrated on only 
a few points (similar to inverse entropy).
If no question has a large shrinkage-cost ratio and
the collision probability is low, then
the expected cost of any query tree must be high.
\begin{lemma} Extension of Lemma 7 \citep{greedy} to non binary queries
and non uniform costs.
For any set $S$ and distribution $v$ over $S$,  
if $\forall i$ $\Delta_i(S, v)/c_i < \Delta / c$, 
then for any $R \subseteq S$ with $R \neq \emptyset$
and any query tree $T$ whose leaves include $R$
\[ C(T, v_R) \geq \frac{c}{\Delta} v(R) (1 - \mathsf{CP}(v_R)) \]
\label{costlma}
\end{lemma}
\begin{proof}
We prove the lemma with induction on $|R|$.  For $|R|=1$, 
$\mathsf{CP}(v_R) = 1$ and the right hand side of the inequality is zero.  
For $R>1$, we lower bound the cost of any query tree
on $R$.  At its root, any query tree chooses some $q_i$ with cost
$c_i$ that divides the version space into $R^j$ for $j \in A$.
Using the inductive hypothesis we can then write the
cost of a tree as
\begin{eqnarray*}
C(T,v_R) & \geq & 
c_i + \sum_{j \in A} v_R(R^j) \frac{c}{\Delta} 
(v(R^j)(1 - \mathsf{CP}(v_{R^j}))) \\
& = & c_i + \frac{c}{\Delta} v(R) 
\sum_{j \in A} (v_R(R^j)^2 - v_R(R^j)^2 \mathsf{CP}(v_{R^j}) ) \\
& = & c_i + \frac{c}{\Delta} v(R) 
(1 - 1 + \sum_{j \in A} v_R(R^j)^2 - \mathsf{CP}(v_R)) \\
\end{eqnarray*}
Here we used 
\begin{align*}
\sum_{j \in A} v_R(R^j)^2 \mathsf{CP}(v_{R^j}) 
= \sum_{j \in A} v_R(R^j)^2 \sum_{r \in R^j} v_{R^j}(r)^2 
= \sum_{r \in R} v_R(r)^2 = \mathsf{CP}(v_R) \\
\end{align*}
We now note $v(R)(1 - \sum_{j \in A} v_R(R^j)^2) = 
v(R) - \sum_{j \in A} v(R^j)^2/v(R) = \Delta_i(R, v)$
\begin{eqnarray*}
C(T,v_R) & \geq & c_i + \frac{c}{\Delta} v(R) (1 - \mathsf{CP}(v_R)) 
- \Delta_i(R, v) \frac{c}{\Delta} \\
& = & \frac{c}{\Delta}v(R) (1 - \mathsf{CP}(v_S)) 
+ \frac{\Delta c_i - \Delta_i(R, v) c}{\Delta} \\
\end{eqnarray*}
Using Corollary \ref{mondeccor}, $\Delta_i(R, v) / c_i \leq 
\Delta_i(S, v) / c_i \leq \Delta / c$, 
so $\Delta c_i - \Delta_i(R, v) c \geq 0$ and therefore
\[ C(R,v_S) \geq \frac{c}{\Delta} v(R) (1 - \mathsf{CP}(v_R)) \]
which completes the induction.
%\qed \end{proof}
\end{proof}
This lower bound on the cost of a tree 
translates into a lower bound on the shrinkage-cost ratio 
of the question chosen by the greedy tree. 
\begin{corollary}Extension of Corollary 8 \citep{greedy} to non binary
queries and non uniform costs.  For any 
$S \subseteq H$ with $S \neq \emptyset$ 
and query tree $T$ whose leaves contain $S$, there
must be a question $q_i$ with  $\Delta_i(S, \pi_S) / c_i
\geq (1 - \mathsf{CP}(\pi_S)) / C(T, \pi_S)$ \label{shrinkcor} \end{corollary}
\begin{proof} 
Suppose this is not the case.  Then there is some
$\Delta/c < (1 - \mathsf{CP}(\pi_S)) / C(T, \pi_S)$ 
such that $\forall i$ $\Delta_i(S, \pi_S)/c_i \leq \Delta/c$. 
By Lemma \ref{costlma} (with $v \triangleq \pi_S$, $R \triangleq S$), 
\begin{align*}
C(T,\pi_S) & \geq \pi_S(S) \frac{c}{\Delta}(1 - \mathsf{CP}(\pi_S)) 
 > \pi_S(S) C(T,\pi_S) = C(T,\pi_S) \\
\end{align*}
which is a contradiction. 
\end{proof}
%\qed \end{proof}

A special case which poses some difficulty for the main proof
is when for some $S \subseteq H$ we have $\mathsf{CP}(\pi_S) > 1/2$.
First note that if $\mathsf{CP}(\pi_S) > 1/2$ one object $h_0$ 
has more than half the mass of $S$.  In the lemma below,
we use $R \triangleq S \setminus \lbrace h_0 \rbrace$.  Also 
let $\delta_i$ be the relative mass of the hypotheses in $R$ that are
distinct from $h_0$ w.r.t.\ question $q_i$.
$ \delta_i \triangleq \pi_R(\lbrace r \in R : q_i(h_0) \neq q_i(r) \rbrace) $
In other words, when question $q_i$ is asked, $R$ is divided into
a set of hypotheses that agree with $h_0$ (these have relative
mass $1-\delta_i$) and a set of hypotheses that disagree with $h_0$ (these
have relative mass $\delta_i$).
\begin{figure}[t]
\centering
\includegraphics[width=.45\textwidth]{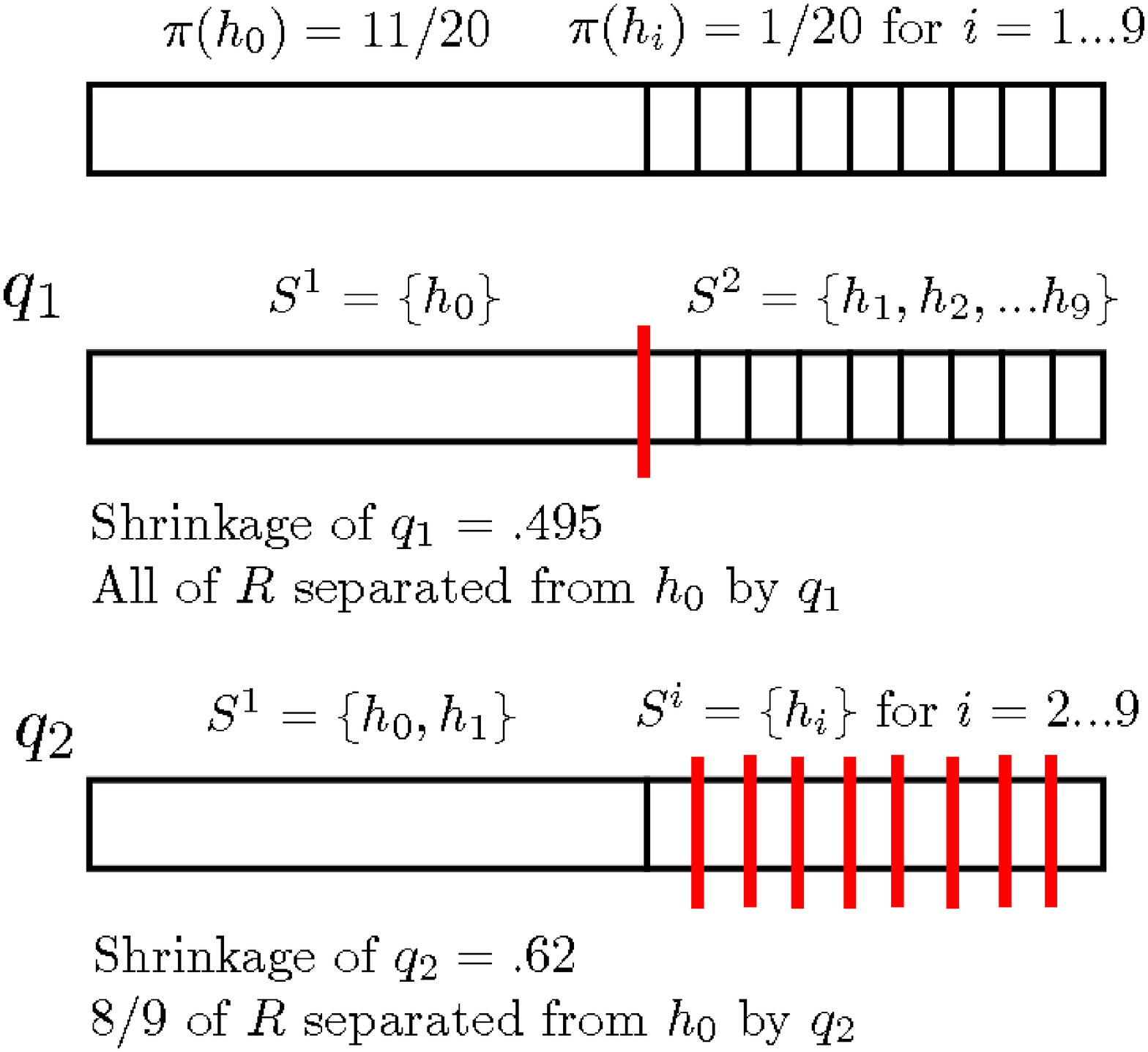}
\includegraphics[width=.45\textwidth]{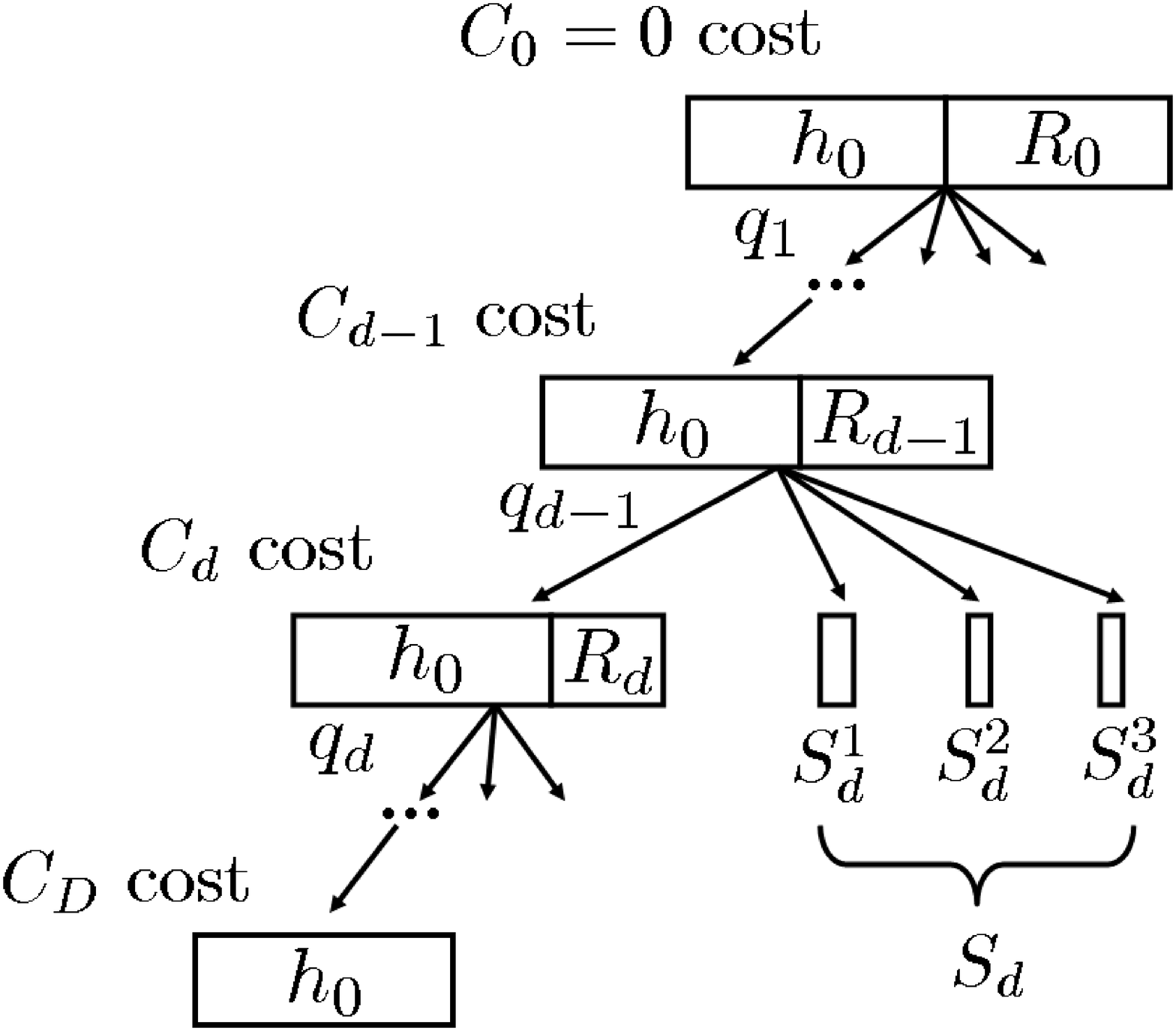}
\caption{Left: Counter example showing that when a single hypothesis $h_0$
contains more than half the mass, the query with maximum shrinkage
is not necessarily the query that separates the most mass from $h_0$.
Right: Notation for this case.}
\label{counterfig}
\end{figure}
\citet{greedy} also treats this as a special case.  However, in
the more general setting treated here the situation is more subtle.  
For yes or no questions, the question
chosen by the greedy query tree is also the question that removes the
most mass from $R$.  In our setting this is not necessarily the case.
The left of Figure \ref{counterfig} shows a counter example.  
However, we can show the fraction of mass removed from $R$ by the
greedy query tree is at least half the fraction removed by any other
question.  Furthermore, to handle costs, we must instead consider the
fraction of mass removed from $R$ \emph{per unit cost}.

In this lemma we use $\pi_{\lbrace h_0 \rbrace}$ to denote the distribution which
puts all mass on $h_0$.  The cost of identifying $h_0$ in a 
tree $T^*$ is then $C^*(h_0) \triangleq C(T^*, \pi_{\lbrace h_0 \rbrace})$.
\begin{lemma}Consider any $S \subseteq H$ and $\pi$ 
with $\mathsf{CP}(\pi_S) > 1/2$ and $\pi(h_0) > 1/2$. 
Let $C^*(h_0) = C(T^*, \pi_{\lbrace h_0 \rbrace})$
for any $T^*$ whose leaves contain $S$. 
Some question $q_i$ has $\delta_i / c_i > 1/C^*(h_0)$.
\label{existcutrlma} \end{lemma}
\begin{proof}
There is always a set of questions indexed by the set $I$ with total cost 
$\sum_{i \in I} c_i \leq C^*(h_0)$ that distinguish $h_0$ from $R$ within S.
In particular, the set of questions used to identify $h_0$ in $T^*$ satisfy
this. Since the set identifies $h_0$, $\sum_{i \in I} \delta_i \geq 1$ 
which implies
\[ \sum_{i \in I} \frac{c_i}{C^*(h_0)} \frac{\delta_i}{c_i} \geq 1/C^*(h_0) \]
Because $c_i/C^*(h_0) \in (0, 1]$ and 
$\sum_{i \in I} c_i/C^*(h_0) \leq 1$, there must be a $q_i$ such that 
$\delta_i/c_i \geq 1 / C^*(h_0)$. 
%\qed \end{proof}
\end{proof}

Having shown that some query always reduces the relative mass of $R$ 
by $1/C^*(h_0)$ per unit cost, we now show that the greedy query tree reduces
the mass of $R$ by at least half as much per unit cost.
\begin{lemma} Consider any $\pi$ and $S \subseteq H$ with $\mathsf{CP}(\pi_S)
> 1/2$, $\pi(h_0) > 1/2$, 
and a corresponding subtree $T^g_S$ in the greedy tree.  Let
$C^*(h_0) = C(T^*, \pi_{\lbrace h_0 \rbrace})$ for any $T^*$ whose
leaves contain $S$.  The question $q_i$ chosen by $T^g_S$ 
has $\delta_i / c_i > 1/(2C^*(h_0))$.
\label{greedycutrlma} \end{lemma}
\begin{proof}
We prove this by showing that the fraction removed from $R$ per unit
cost by the greedy query tree's question is at least half that of any
other question.  Combining this with Lemma \ref{existcutrlma}, we get the
desired result.  

We can write the shrinkage of $q_i$ in terms
of $\delta_i$.  Here let $A' \triangleq A \setminus \lbrace q_i(h_0) \rbrace$.
Since $\pi(S^{q_i(h_0)}) = \pi(h_0) + \left( \pi(S) - \delta_i \pi(R) \right)$, 
and $\pi(S) - \pi(S^{q_i(h_0)}) = \delta_i \pi(R)$, we have that 
\begin{align*}
\Delta_i(S, \pi_S) = 
(\pi_S(h_0) + (1-\delta_i)\pi_S(R)) \delta_i \pi_S(R) 
 +  \sum_{j \in A'} \pi_S(S^j)(\pi_S(S)-\pi_S(S^j))
\end{align*}
We use $\sum_{j \in A'} \pi_S(S^j) = \delta_i \pi_S(R)$.

We can then upper bound the shrinkage using 
$\pi_S(S)-\pi_S(S^j) \leq 1 $
\begin{eqnarray*}
\Delta_i(S, \pi_S) 
& \leq & 
(\pi_S(h_0) + (1-\delta_i)\pi_S(R)) \delta_i \pi_S(R) + \delta_i \pi_S(R) 
\leq  2 \delta_i \pi_S(R) 
\end{eqnarray*}
and lower bound the shrinkage using 
$\pi_S(h_0) > 1/2$ and
$\pi_S(S)-\pi_S(S^j) > \pi_S(h_0) + (1-\delta_i)\pi_S(R)$ for any $j \in A'$
\begin{eqnarray*}
\Delta_i(S, \pi_S) 
&\geq & 2(\pi_S(h_0) + (1-\delta_i)\pi_S(R)) \delta_i \pi_S(R)  
\geq \delta_i \pi_S(R) 
\end{eqnarray*}

Let $q_i$ be any question and $q_j$ be the question chosen by the greedy tree
giving  $\Delta_j(S, \pi_S) / c_j \geq \Delta_i(S, \pi_S) / c_i$.
Using the upper and lower bounds we derived, we then know 
$2\delta_j \pi_S(R) / c_j \geq \delta_i \pi_S(R) / c_i$ and can conclude
$2\delta_j /c_j \geq \delta_i/c_i$.  
Combining this with Lemma \ref{existcutrlma},
$\delta_j /c_j \geq 1/(2C^*(h_0)$.  
%\qed \end{proof}
\end{proof}

The main theorem immediately follows from the next theorem.
\begin{theorem}If $T^*$ is any query tree for $\pi$ and $T^g$ is the
greedy query tree for $\pi$, then for any $S \subseteq H$
corresponding to the subtree $T^g_S$ of $T^g$,
\[ C(T^g_S, \pi_S) \leq 
12 C(T^*, \pi_S) \ln \frac{\pi(S)}{\min_{h \in S}\pi(h)} \]
\label{submainthm}
\end{theorem}
\begin{proof}
In this proof we use $C^*(S)$ as a short hand for $C(T^*,\pi_S)$.
Also, we use $\min(S)$ for $\min_{s \in S}\pi(S)$.  We proceed with
induction on $|S|$.  For $|S|=1$, $C(T^g_S, \pi_S)$ is zero and the
claim holds.  For $|S| > 1$, we consider two cases.

\noindent \underline{\textbf{Case one:} $\mathsf{CP}(\pi_S) \leq 1/2$}

At the root of $T^g_S$, the greedy query tree chooses some $q_i$
with cost $c_i$ that reduces the version space to $S^j$ when
$q_i(h^*)=j$.  Let 
$\pi(S^{+}) \triangleq \max \lbrace \pi(S^j) : j \in A \rbrace$
Using the inductive hypothesis
\begin{eqnarray*}
C(T^g_S, \pi_S) 
& = & c_i + \sum_{j \in A} \pi_S(S^j) C(T_{S^j}, \pi_{S^j}) \\
& \leq & c_i + \sum_{j \in A} 12 \pi_S(S^j) C^*(S^j)
\ln \frac{\pi(S^j)}{\min(S^j)} \\
& \leq & c_i + 12(\sum_{j \in A} \pi_S(S^j) C^*(S^j)) 
\ln \frac{\pi(S^{+})}{\min(S)} 
\end{eqnarray*}
Now using Lemma \ref{decomplem},
$\pi(S^{+})=\pi(S) \pi_S(S^+)$, and then $ln(1-x) \leq -x$
\begin{eqnarray*}
C(T^g_S, \pi_S)
& \leq & c_i + 12 C^*(S) \ln \frac{\pi(S)}{\min(S)} 
+ 12 C^*(S) \ln \pi_S(S^{+}) \\
& \leq & c_i + 12 C^*(S) \ln \frac{\pi(S)}{\min(S)} 
- 12 C^*(S) (1 - \pi_S(S^{+}))
\end{eqnarray*}
$\pi_S(S^+) \geq
\sum_{j \in A} \pi_S(S^j)^2$ because this sum is an expectation
and $\forall_j$ $\pi_S(S^+) \geq \pi_S(S^j)$.
From this follows
\begin{eqnarray*}
C(T^g_S, \pi_S)
& \leq  & c_i + 12 C^*(S) \ln \frac{\pi(S)}{\min(S)}
- 12 C^*(S) (1 - \sum_{j \in A}\pi_S(S^j)^2) \\
& = & c_i + 12 C^*(S) \ln \frac{\pi(S)}{\min(S)} 
 - 12 C^*(S) c_i \frac{(1 - \sum_{j \in A}\pi_S(S^j)^2))}{c_i} 
\end{eqnarray*}
$(1 - \sum_{j \in A}\pi_S(S^j)^2)$ is 
$\Delta_i(S, \pi_S)$, so by Corollary \ref{shrinkcor} and using
$\mathsf{CP}(\pi_S) \leq 1/2$
\begin{eqnarray*}
C(T^g_S, \pi_S)
& \leq  & c_i + 12 C^*(S) \ln \frac{\pi(S)}{\min(S)} 
 - 12 C^*(S) c_i \frac{1 - \mathsf{CP}(\pi_S)}{C^*(S)} \\
& = & c_i+12 C^*(S) \ln \frac{\pi(S)}{\min(S)} - 12(1 - \mathsf{CP}(\pi_S))c_i \\
& \leq & 12 C^*(S) \ln \frac{\pi(S)}{\min(S)}
\end{eqnarray*}
which completes this case.

\noindent \underline{\textbf{Case two:} $\mathsf{CP}(\pi_S) > 1/2$}

The hypothesis with more than half the mass, $h_0$, 
lies at some depth $D$ in the greedy tree $T^g_S$.  
Counting the root of $T^g_S$ as depth $0$, $D \geq 1$. 
At depth $d>0$, let $q_0, q_1, ... q_{d-1}$ be the questions
asked so far, $c_0, c_1, ... c_{d-1}$ be the costs of these questions, 
and $C_d = \sum_{i=0}^{d-1} c_i$ be the total cost incurred.  At
the root, $C_0=0$.  

At depth $d < D$, we define $R_d$ to be the set of objects other than
$h_0$ that are still in the version space along the path to $h_0$.
$R_0 \triangleq S \setminus \lbrace h_0 \rbrace$
and for $d > 0$
$R_d \triangleq 
R_{d-1} \setminus \lbrace h : q_{d-1}(h) \neq q_{d-1}(h_0) \rbrace $.
In other words, $R_d$ is $R_{d-1}$ with the objects that
disagree with $h_0$ on $q_{d-1}$ removed.  
All of the objects in $R_d$ have the same response as $h_0$ 
for $q_0, q_1, ..., q_{d-1}$.  The right of Figure \ref{counterfig} 
shows this case.

We first bound the mass  remaining in $R_d$ as a function of the label cost 
incurred so far. 
For $d>0$, using Lemma \ref{greedycutrlma}, 
\begin{eqnarray*}
\pi(R_d) & \leq & \pi(R_0) \prod_{i=0}^{d-1} (1 - \frac{c_i}{2C^*(h_0)}) 
\leq \pi(R_0) e^{-C_d / (2C^*(h_0))} 
\end{eqnarray*}
Using this bound, we can bound $C_D$, the cost of identifying
$h_0$ (i.e. $C(T^g_S, h_0)$). 
First note that $\pi(R_{D-1}) \geq \min(R_0)$ since at
least one object is left in $R_{D-1}$.  Combining this with the
upper bound on the mass of $R_d$, we have if $D-1>0$.
\[ C_{D-1} \leq 2 C^*(h_0) \ln(\pi(R_0)/\min(R_0)) \]
This clearly also holds if $D-1=0$, since, $C_0=0$.
We now only need to bound the cost of the final question (the question
asked at level $D-1$).  If the final question had cost greater than
$2C^*(h_0)$, then by Lemma \ref{greedycutrlma}, this question would
reduce the mass of the set containing $h_0$ to less than $\pi(h_0)$.
This is a contradiction, so the final question must have cost no greater
than $2C^*(h_0)$.  
\begin{eqnarray*}
C_{D} & \leq & 2 C^*(h_0) \ln\frac{\pi(R_0)}{\min(R_0)} + 2C^*(h_0) \\
\end{eqnarray*}

We use $A_{d-1}' \triangleq A \setminus q_{d-1}(h_0)$.  
Let $s \in S_d^j$ be the set of objects removed from $R_{d-1}$ with the
question at depth $d-1$ such that $q_{d-1}(s)=j$, that is
$R_{d-1}=R_d+\bigcup_{j \in A_{d-1}'}S_d^j$. 
Let $S_d = \bigcup_{j \in A_{d-1}'} S_d^j$.
The right of Figure \ref{counterfig} illustrates this notation.
A useful variation of Lemma \ref{decomplem} we use
in the following is that for
$S = S^1 \cup S^2$ and $S^1 \cap S^2 = \emptyset$,
$\pi(S) C^*(S) = \pi(S^1) C^*(S^1) + \pi(S^2) C^*(S^2)$.

We can write
\begin{eqnarray*}
\pi(S)C(T^g_S, \pi_S)
& \stackrel{a}{=} & \pi(h_0) C_D + \sum_{d=1}^{D} \sum_{j \in A_{d-1}'}
\pi(S_d^j) (C_d + C(T_{S_d^j}, \pi_{S_d^j})) \\
& \stackrel{b}{\leq} & \pi(h_0) C_D + \sum_{d=1}^{D} \pi(S_d) C_d  
 + \sum_{d=1}^{D} \sum_{j \in A_{d-1}'}
\pi(S_d^j) 12 C^*(S_d^j) \ln \frac{\pi(S_d^j)}{\min(S_d^j)} \\
& \stackrel{c}{\leq} & \pi(h_0) C_D + 
\pi(R_0) C_D + 12 \pi(R_0) C^*(R_0) \ln \frac{\pi(R_0)}{\min(R_0)} \\
& \stackrel{d}{\leq} & 2 \pi(h_0) C_D + 
12 \pi(R_0) C^*(R_0) \ln \frac{\pi(R_0)}{\min(R_0)} \\
\end{eqnarray*}
Here a) decomposes the total cost into the cost of identifying
$h_0$ and the cost of each branch leaving the path to $h_0$.
For each of these branches the total cost is the cost incurred so far
plus the cost of the tree rooted at that branch.
b) uses the inductive hypothesis, c) uses 
$\forall_{i,j} S_i \cap S_j = \emptyset$ and $\bigcup_d S_d = R_0$,
and d) uses $\pi(R_0)<\pi(h_0)$.  Continuing
\begin{eqnarray*}
\pi(S)C(T^g_S, \pi_S)
& \stackrel{a}{\leq} & 4 \pi(h_0) C^*(h_0) 
(\ln \frac{\pi(R_0)}{\min(R_0)} + 1)
+ 12 \pi(R_0) C^*(R_0) \ln \frac{\pi(R_0)}{\min(R_0)} \\
& \stackrel{b}{\leq} &  4 \pi(h_0) C^*(h_0) 
(\ln \frac{\pi(S)}{\min(S)} + 1)
+ 12 \pi(R_0) C^*(R_0) \ln \frac{\pi(S)}{\min(S)} \\
\end{eqnarray*}
where a) uses our bound on $C_D$ and b) uses $R_0 \subset S$.  Finally
\begin{eqnarray*}
\pi(S)C(T^g_S, \pi_S) 
& \leq &  12 \pi(h_0) C^*(h_0) \ln \frac{\pi(S)}{\min(S)} 
 + 12 \pi(R_0) C^*(R_0) \ln \frac{\pi(S)}{\min(S)} \\
& = & \pi(S) 12 C^*(S) \ln \frac{\pi(S)}{\min(S)} 
\end{eqnarray*}
where we use $\pi(S) > 2 \min(S)$ and therefore 
$\ln \frac{\pi(S)}{\min(S)} > \ln 2 > .5$.  Dividing both sides by $\pi(S)$
gives the desired result.
\end{proof}

\section{Distribution Independent Bound}

We now show the dependence on $\pi$ can be removed using a variation
of the rounding trick used by \citet{optimalsplit} and
\citet{decisiontreesentity}.  The intuition behind this trick is that
we can round up small values of $\pi$ to obtain a distribution $\pi'$
in which $\ln (1 / \min_{h \in H} \pi'(h)) = O(\ln n)$ while ensuring
that for any tree $T$, $C(T, \pi)/C(T, \pi')$ is bounded above and
below by a constant.  
Here $n=|H|$.
When the greedy algorithm is applied to this
rounded distribution, the resulting tree gives an $O(\log n)$
approximation to the optimal tree for the original distribution.  In
our cost sensitive setting, the intuition remains the same, but
the introduction of costs changes the result.

Let $c_{\max} \triangleq \max_{i} c_i$ and $c_{\min} \triangleq
\min_{i} c_i$.  In this discussion, we consider \emph{irreducible
  query trees}, which we define to be query trees which contain only
questions with non-zero shrinkage.  Greedy query trees will always
have this property as will optimal query trees.  This property let's
us assume any path from the root to a leaf has at most $n$ nodes with
cost at most $c_{\max}n$ because at least one hypothesis is eliminated
by each question.  Define $\pi'$ to be the distribution obtained from
$\pi$ by adding $c_{\min}/(c_{\max}n^3)$ mass to any hypothesis $h$
for which $\pi(h) < c_{\min}/(c_{\max}n^3)$.  Subtract the
corresponding mass from a single hypothesis $h_j$ for which $\pi(h_j)
\geq 1/n$ (there must at least one such hypothesis).  By construction,
we have that $\min_{i} \pi'(h_i) \geq c_{\min}/(c_{\max}n^3)$.  We can
also bound the amount by which the cost of a tree changes as a result
of rounding
\begin{lemma} For any irreducible query tree $T$ and $\pi$, 
\[ \frac{1}{2}C(T, \pi) \leq C(T, \pi') \leq \frac{3}{2} C(T, \pi) \] 
\label{costdistlma}
\end{lemma}
\begin{proof}
For the first inequality, let $h'$ be the hypothesis we subtract
mass from when rounding.  The cost to identify $h'$,
$c_T(h')$ is at most $c_{\max} n$.
Since we subtract at most
$c_{\min}/(c_{\max}n^2)$ mass and $c_T(h') \leq c_{\max}n$, we then have
\begin{align*}
C(T, \pi') & \geq C(T, \pi) - \frac{c_{\min}}{c_{\max}n^2} c_T(h') 
\geq C(T, \pi) - \frac{c_{\min}}{n} 
\geq \frac{1}{2} C(T, \pi)
\end{align*}
The last step uses and $C(T, \pi) > c_{\min}$ and $n>2$.
For the second inequality, we add at most $c_{\min}/(c_{\max}n^3)$
mass to each hypothesis and $\sum_{h} c_T(h) < c_{\max} n^2$, so
\begin{align*}
C(T, \pi') & \leq C(T, \pi) + \sum_{h \in H} \frac{c_{\min}}{c_{\max}n^3} c_T(h) 
\leq C(T, \pi) + \frac{c_{\min}}{n} 
\leq \frac{3}{2} C(T, \pi)  
\end{align*}
The last step again uses $C(T, \pi) > c_{\min}$ and $n>2$
%\qed \end{proof}
\end{proof}

We can finally give a bound on the greedy algorithm applied to $\pi'$,
in terms of $n$ and $c_{max}/c_{min}$
\begin{theorem}  
For any $\pi$ the greedy query tree $T^g$ for $\pi'$ has cost 
at most 
\[ C(T^g, \pi) \leq O(C^* \ln (n \frac{c_{\max}}{c_{\min}})) \]
where $C^* \triangleq \min_T C(T, \pi)$.  \label{distinthm} \end{theorem}
\begin{proof}
Let $T'$ be an optimal tree for $\pi'$ and $T^*$ be an optimal
tree for $\pi$.
Using Theorem \ref{mainthm},  $\min_{i} \pi'(h_i) \geq c_{\min}/(c_{\max}n^3)$, 
and  Lemma \ref{costdistlma}.
\begin{align*}
C(T^g, \pi) \leq & 2C(T^g, \pi') 
\leq 72 C(T', \pi') \ln (n \frac{c_{\max}}{c_{\min}}) \\
 \leq & 72 C(T^*, \pi') \ln (n \frac{c_{\max}}{c_{\min}}) 
 \leq  108 C(T^*, \pi) \ln (n \frac{c_{\max}}{c_{\min}})
\end{align*}
%\qed \end{proof}
\end{proof}

\section{$\epsilon$-Approximate Algorithm}

Some of the non traditional active learning scenarios involve a large
number of possible questions.  For example, in the batch active
learning scenario we describe, there may be a question corresponding
to every subset of single data point questions.  In these scenarios,
it may not be possible to exactly find the question with largest
shrinkage-cost ratio.  It is not hard
to extend our analysis to a strategy that at each step finds a
question $q_i$ with 
\[ \Delta_i(S, \pi_S) / c_i \geq (1 - \epsilon)
\max_j \Delta_j(S, \pi_S) / c_j \]
for $\epsilon \in [0, 1)$.  
We call this the $\epsilon$-approximate
cost sensitive greedy algorithm.  Algorithm \ref{approxgreedyalg} outlines this
strategy.  We show $\epsilon > 0$ only introduces an
$1/(1 - \epsilon)$ factor into the bound.  \citet{optimalsplit}
report a similar extension to their result.

\begin{algorithm}[t]
\caption{$\epsilon$-Approximate Cost Sensitive Greedy Algorithm}
\begin{algorithmic}[1]
\STATE $S \Leftarrow H$
\REPEAT
\STATE Find $i$ so
$\Delta_i(S, \pi_S) / c_i > (1 - \epsilon) \max_j \Delta_j(S, \pi_S) / c_j$
\STATE $S \Leftarrow \lbrace s \in S : q_i(s) = q_i(H) \rbrace$
\UNTIL{$|S| = 1$}
\end{algorithmic}
\label{approxgreedyalg}
\end{algorithm}

\begin{theorem}
For any $\pi$ the $\epsilon$-approximate 
greedy query tree $T$ has cost at most 
\[ C(T, \pi) \leq (12 / (1-\epsilon)) C^* \ln 1/(\min_{h \in H} \pi(h)) \] 
where $C^* = \min_T C(T, \pi)$. \end{theorem}

This theorem follows from extensions of Corollary \ref{shrinkcor},
Lemma \ref{greedycutrlma}, and Theorem \ref{submainthm}.  
The proofs are straightforward, but we outline them below for completeness.
It is also straightforward to derive a similar extension of
Theorem \ref{distinthm}.
This corollary follows directly from Corollary \ref{shrinkcor} and
the $\epsilon$-approximate algorithm.

\begin{corollary}For any 
$S \subseteq H$ and query tree $T$ whose leaves contain $S$, the
question $q_i$ chosen by an $\epsilon$-approximate query
tree has $\Delta_i(S, \pi_S) / c_i
\geq (1-\epsilon) (1 - \mathsf{CP}(\pi_S)) / C(T, \pi_S)$ 
\label{epshrinkcor} \end{corollary}

This lemma extends Lemma \ref{greedycutrlma} to the approximate case.

\begin{lemma} Consider any $\pi$ and $S \subseteq H$ with $\mathsf{CP}(\pi_S) >
1/2$ and a corresponding subtree $T^{\epsilon}_S$ in an $\epsilon$-approximate 
greedy tree. Let $C^*(h_0) = C(T^*, \pi_{\lbrace h_0 \rbrace})$ for any $T^*$.
The question $q_i$ chosen by $T^{\epsilon}_S$ 
has $\delta_i / c_i > (1-\epsilon)/(2C^*(h_0))$.
. \label{approxgreedycutrlma} \end{lemma}

\begin{proof}
The proof follows that of Lemma \ref{greedycutrlma}.  We show
the fraction of $R$ removed for unit cost by the $\epsilon$-approximate
greedy tree is at least $(1-\epsilon) / 2$ that of any other question.
Using Lemma \ref{existcutrlma} the result then follows. Let $q_i$ be any
question and $q_j$ be the
question chosen by an $\epsilon$-approximate greedy tree.  
$\Delta_j(S, \pi_S) / c_j \geq (1-\epsilon) \Delta_i(S, \pi_S) / c_i$.
Using upper and lower bounds from Lemma \ref{greedycutrlma}, we then know 
$2\delta_j \pi_S(R) / c_j \geq (1-\epsilon) 
\delta_i \pi_S(R) / c_i$ and can conclude
$2\delta_j / (c_j (1 - \epsilon)) \geq \delta_i/c_i$.  The lemma then
follows from Lemma \ref{existcutrlma}. 
%\qed \end{proof}
\end{proof}

\begin{theorem}If $T^*$ is any query tree for $\pi$ and $T^{\epsilon}$ is an
$\epsilon$-approximate greedy query tree for $\pi$,
then for any $S \subseteq H$ corresponding to the subtree $T^{\epsilon}_S$ of 
$T^{\epsilon}$,
\[ C(T^{epsilon}_S, \pi_S) \leq 
\frac{12}{(1 - \epsilon)} C(T^*, \pi_S) \ln \frac{\pi(S)}{\min_{h \in S}\pi(h)} \]
\end{theorem}

\begin{proof}
The proof follows very closely that of Theorem \ref{submainthm}, and we
use the same notation.  We again use induction on $|S|$, and the
base case holds trivially.

\noindent \underline{\textbf{Case one:} $\mathsf{CP}(\pi_S) \leq 1/2$}

Using the inductive hypothesis and
the same steps as in Theorem \ref{submainthm} one can show
\begin{eqnarray*}
C(T^{\epsilon}_S, \pi_S)
& \leq & c_i + \frac{12}{(1-\epsilon)} C^*(S) \ln \frac{\pi(S)}{\min(S)} -
\frac{12}{(1-\epsilon)} C^*(S) c_i \frac{(1 - \sum_{j \in A}\pi_S(S^j)^2))}{c_i} \\
\end{eqnarray*}
$(1 - \sum_{j \in A}\pi_S(S^j)^2)$ is 
$\Delta_i(S, \pi_S)$, so using Corollary \ref{epshrinkcor} and
$\mathsf{CP}(\pi_S) \leq 1/2$.
\begin{eqnarray*}
C(T^{\epsilon}_S, \pi_S)
& \leq & c_i + \frac{12}{(1-\epsilon)} C^*(S) \ln \frac{\pi(S)}{\min(S)}
- \frac{12}{(1-\epsilon)} C^*(S) c_i (1-\epsilon) 
\frac{1 - \mathsf{CP}(\pi_S)}{C^*(S)} \\
& = & c_i + \frac{12}{(1-\epsilon)} C^*(S) \ln \frac{\pi(S)}{\min(S)}
- 12 (1 - \mathsf{CP}(\pi_S)) c_i \\
& \leq & \frac{12}{(1-\epsilon)} C^*(S) \ln \frac{\pi(S)}{\min(S)} \\
\end{eqnarray*}
which completes this case.

\noindent \underline{\textbf{Case two:} $\mathsf{CP}(\pi_S) > 1/2$}

Using Lemma \ref{approxgreedycutrlma} and the same steps and notation as in 
Theorem \ref{submainthm}
\begin{eqnarray*}
\pi(R_d) & \leq & \pi(R_0) e^{-C_d (1-\epsilon) / (2C^*(h_0))} 
\end{eqnarray*}

Using this bound, we can again bound $C_D$, the cost of identifying
$h_0$. 
\begin{eqnarray*}
C_{D} & \leq & \frac{2}{(1-\epsilon)} C^*(h_0) \ln\frac{\pi(R_0)}{\min(R_0)} + 
\frac{2}{(1-\epsilon)} C^*(h_0) \\
\end{eqnarray*}

The remainder of the case follows the same steps as Theorem \ref{submainthm}. 
%\qed \end{proof}
\end{proof}

\section{Related Work}

\begin{table*}[t]
\begin{center}
\begin{small}
\begin{tabular}{|c|llll|}
\hline 
& $k>2$ & Non uniform $c_i$ & Non uniform $\pi$ & Result \\
\hline 
\citet{optimalsplit} & Y & N & Y & $O(\log n)$ \\
\citet{greedy} & N & N & Y & $O(\log (1/\min_h \pi(h)))$ \\
\citet{approxoptimal} & N & Y & N & $O(\log n)$ \\
\citet{decisiontreesentity} & Y & N & Y & $O(\log k \log n)$ \\
\citet{approxdecision} & Y & N & N & $O(\log n)$ \\
This paper & Y & Y & Y & $O(\log (1/\min_h \pi(h)))$ \\
This paper & Y & Y & Y & $O(\log (n \max_i c_i / \min_i c_i))$ \\
\hline 
\end{tabular}
\end{small}
\end{center}
\caption{Summary of approximation ratios achieved by 
related work.  Here $n$ is the number of objects,
$k$ is the number of possible responses, $c_i$ are the question costs,
and $\pi$ is the distribution over objects.}
\label{prevwork}
\end{table*}

Table \ref{prevwork} summarizes previous results analyzing greedy
approaches to this problem.  A number of these results were derived
independently in different contexts.  Our work gives the first
approximation result for the general setting in which there are more
than two possible responses to questions, non uniform question costs,
and a non uniform distribution over objects.  We give bounds for two
algorithms, one with performance independent of the query costs and
one with performance independent of the distribution over objects.
Together these two bounds match all previous bounds for
less general settings.  We also note that \citet{optimalsplit} only mention
an extension to non binary queries (Remark 1), and our work is
the first to give a full proof of an $O(\log n)$ bound for the case of non
binary queries and non uniform distributions over objects..

Our work and the work we extend are examples of exact active learning.
We seek to exactly identify a target hypothesis from a finite set
using a sequence of queries.  Other work considers active learning
where it suffices to identify with high probability a hypothesis close
to the target hypothesis \citep{generalagnostic, importance}.  The exact and
approximate problems can sometimes be related \citep{teaching}.

Most theoretical work in active learning assumes unit costs and simple
label queries.  An exception, \citet{costcomp}
also considers a general learning framework in
which queries are arbitrary and have known costs associated with them.
In fact, the setting used by \citet{costcomp} is more general in that 
questions are allowed to have more than one valid answer for each
hypothesis.  \citet{costcomp} gives worst-case upper and lower bounds in
terms of a quantity called the General Identification Cost and related
quantities.  There are interesting parallels between our average-case
analysis and this worst-case result.

Practical work incorporating costs in active learning 
\citep{activereal, activeroi} has also considered methods
that maximize a benefit-cost ratio similar in spirit to the
method used here.  However, \citet{activereal} suggests
this strategy may not be sufficient for practical cost savings.

\section{Implications}

We briefly discuss the implications of our result in terms of the
motivating applications.  

For the active learning applications, our result shows that the
cost-sensitive greedy algorithm approximately minimizes cost compared
to any other deterministic strategy using the same set of queries.
For the the batch learning setting, if we create a question corresponding
to each subset of the dataset, then the resulting greedy strategy does
approximately as well as any other algorithm that makes a sequence of
batch label queries.  This result holds no matter how we assign costs
to different queries although restrictions may need to be made in
order to ensure computing the greedy strategy is feasible.  Similarly,
for the partial label query setting, the greedy strategy is
approximately optimal compared to any other active learning algorithm
using the same set of partial label queries. 

In the information retrieval domain, our result shows that when the
cost of a question is set to be the computational cost of determining
which branch an object is in, the resulting greedy query tree is
approximately optimal with respect to expected search time.  Although
the result only holds for expected search time and for searches for
objects in the tree (i.e. point location queries), the result is very
general.  In particular, it makes no restriction on the type of splits
(i.e. questions) used in the tree, and the result therefore applies to
many kinds of search trees.  In this application, our result
specifically improves previous results by allowing for arbitrary mixing of
different kinds of splits through the use of costs.

Finally, in the compression domain, our result shows gives a bound
on expected code length for top-down greedy code construction.  Top-down
greedy code construction is known to be suboptimal, but our result shows
it is approximately optimal and generalizes previous bounds.

\section{Open Problems}

\citet{decisiontreesentity} show it is NP-hard to approximate the
optimal query tree within a factor of $\Omega (\log n)$ for binary
queries and non uniform $\pi$.  This hardness result is with respect
to the number of objects.  Some open questions remain.  For the more
general setting with non uniform query costs, is there an algorithm
with an approximation ratio independent of both $\pi$ and $c_i$?  The
simple rounding technique we use seems to require dependence on $c_i$, but a
more advanced method could avoid this dependence.  Also, can the
$\Omega (\log n)$ hardness result be extended to the more restrictive
case of uniform $\pi$?  It would also be interesting to extend our
analysis to allow for questions to have more than one valid answer for
each hypothesis.  This would allow queries which ask for a positively
labeled example from a set of examples.  Such an extension appears
non trivial, as a straightforward extension assuming the given answer
is randomly chosen from the set of valid answers produces a tree in which
the mass of hypotheses is split across multiple branches, affecting
the approximation.

Much work also remains in the analysis of 
other active learning settings with general queries and costs.  
Of particular practical interest are extensions to agnostic
algorithms that converge to the correct hypothesis under no
assumptions \citep{generalagnostic, importance}.  Extensions to 
treat label costs, partial label queries, and batch mode active learning
are all of interest, and these learning algorithms could potentially
be extended to treat these three sub problems at once using a similar
setting.  

For some of these algorithms, even without modification
we can guarantee the method does no worse than passive learning with
respect to label cost.  In particular, \citet{generalagnostic} and
\citet{importance} both give algorithms that iterate through $T$
examples, at each step requesting a label with probability $p_t$.
These algorithm are shown to not do much worse (in terms of
generalization error) than the passive algorithm which requests every
label.  Because the algorithm queries for labels for a subset of $T$
i.i.d. examples, the label cost of the algorithm is also no worse than
the passive algorithm requesting $T$ random labels.  It remains an
open problem however to show these algorithms can do better than
passive learning in terms of label cost (most likely this will require
modifications to the algorithm or additional assumptions).
\begin{small}
\bibliography{activemulticost}  \bibliographystyle{abbrvnat}

\begin{thebibliography}{12}
\providecommand{\natexlab}[1]{#1}
\providecommand{\url}[1]{\texttt{#1}}
\expandafter\ifx\csname urlstyle\endcsname\relax
  \providecommand{\doi}[1]{doi: #1}\else
  \providecommand{\doi}{doi: \begingroup \urlstyle{rm}\Url}\fi

\bibitem[Adler and Heeringa(2008)]{approxoptimal}
M.~Adler and B.~Heeringa.
\newblock Approximating optimal binary decision trees.
\newblock In \emph{APPROX}, 2008.

\bibitem[Beygelzimer et~al.(2009)Beygelzimer, Dasgupta, and
  Langford]{importance}
A.~Beygelzimer, S.~Dasgupta, and J.~Langford.
\newblock Importance weighted active learning.
\newblock In \emph{ICML}, 2009.

\bibitem[Chakaravarthy et~al.(2007)Chakaravarthy, Pandit, Roy, Awasthi, and
  Mohania]{decisiontreesentity}
V.~T. Chakaravarthy, V.~Pandit, S.~Roy, P.~Awasthi, and M.~Mohania.
\newblock Decision trees for entity identification: approximation algorithms
  and hardness results.
\newblock In \emph{PODS}, 2007.

\bibitem[Chakaravarthy et~al.(2009)Chakaravarthy, Pandit, Roy, and
  Sabharwal]{approxdecision}
V.~T. Chakaravarthy, V.~Pandit, S.~Roy, and Y.~Sabharwal.
\newblock Approximating decision trees with multiway branches.
\newblock In \emph{ICALP}, 2009.

\bibitem[Cover and Thomas(2006)]{coverthomas}
T.~M. Cover and J.~A. Thomas.
\newblock \emph{Elements of Information Theory 2nd Edition}.
\newblock Wiley-Interscience, 2006.

\bibitem[Dasgupta(2004)]{greedy}
S.~Dasgupta.
\newblock Analysis of a greedy active learning strategy.
\newblock In \emph{NIPS}, 2004.

\bibitem[Dasgupta et~al.(2007)Dasgupta, Hsu, and Monteleoni]{generalagnostic}
S.~Dasgupta, D.~Hsu, and C.~Monteleoni.
\newblock A general agnostic active learning algorithm.
\newblock In \emph{NIPS}, 2007.

\bibitem[Haertel et~al.(2008)Haertel, Sepppi, Ringger, and Carroll]{activeroi}
R.~Haertel, K.~D. Sepppi, E.~K. Ringger, and J.~L. Carroll.
\newblock Return on investment for active learning.
\newblock \emph{NIPS Workshop on Cost-Sensitive Learning}, 2008.

\bibitem[Hanneke()]{costcomp}
S.~Hanneke.
\newblock The cost complexity of interactive learning.
\newblock 2006. Unpublished.
  \url{http://www.cs.cmu.edu/~shanneke/docs/2006/cost-complexity-working-notes%
.pdf}.

\bibitem[Hanneke(2007)]{teaching}
S.~Hanneke.
\newblock Teaching dimension and the complexity of active learning.
\newblock In \emph{COLT}, 2007.

\bibitem[Kosaraju et~al.(1999)Kosaraju, Przytycka, and Borgstrom]{optimalsplit}
S.~R. Kosaraju, T.~M. Przytycka, and R.~Borgstrom.
\newblock On an optimal split tree problem.
\newblock In \emph{WADS}, 1999.

\bibitem[Settles et~al.(2008)Settles, Craven, and Friedland]{activereal}
B.~Settles, M.~Craven, and L.~Friedland.
\newblock Active learning with real annotation costs.
\newblock \emph{NIPS Workshop on Cost-Sensitive Learning}, 2008.

\end{thebibliography}
\end{small}
\end{document}